\documentclass[journal, onecolumn, draftclsnofoot, 12pt]{IEEEtran}
\IEEEoverridecommandlockouts
\usepackage{cite}
\usepackage{amsmath,amssymb,amsfonts, amsthm}

\usepackage{graphicx}
\usepackage{textcomp}
\usepackage{xcolor}
\usepackage{booktabs}
\usepackage{refstyle}
\usepackage{xspace}

\usepackage{algorithm}
\usepackage{algpseudocode}
\usepackage{pgfplots}
\usepackage{comment}
\usepackage[hidelinks]{hyperref}
\usepackage{multirow}

\newtheorem{defn}{Definition}
\newtheorem{thm}{Theorem}
\newtheorem{lem}{Lemma}

\newref{eq}{refcmd=(\ref{#1})}
\newref{sec}{name=Section~}
\newref{subsec}{name=Section~}
\newref{thm}{name=Theorem~}
\newref{lem}{name=Lemma~}
\newref{cor}{name=Corollary~}
\newref{fig}{name=Fig.~}
\newref{def}{name=Definition~}
\newref{algor}{name=Algorithm~} 
\newref{rem}{name=Remark~}

\providecommand{\TextSelim}[1]{#1}
\providecommand{\TextBurak}[1]{#1}

\def\BibTeX{{\rm B\kern-.05em{\sc i\kern-.025em b}\kern-.08em
    T\kern-.1667em\lower.7ex\hbox{E}\kern-.125emX}}
    
\DeclareMathOperator*{\argmax}{arg\,max}

\renewcommand{\vec}[1]{\boldsymbol{#1}}

\newcommand{\real}{\mathbb{R}}
\newcommand{\card}[1]{\vert{#1}\vert}

\algnewcommand\algorithmicforeach{\textbf{for each}}
\algdef{S}[FOR]{ForEach}[1]{\algorithmicforeach\ #1\ \algorithmicdo}

\pgfplotsset{compat=1.17}

\begin{document}

\title{Over-the-Air Ensemble Inference \\with Model Privacy
\thanks{
The authors are with Department of Electrical and Electronic Engineering, Imperial College London, UK. Deniz Gündüz is also with Department of Engineering “Enzo Ferrari”, University of Modena and Reggio Emilia (UNIMORE), Italy. Email: \{s.yilmaz21, b.hasircioglu18, d.gunduz\}@imperial.ac.uk

\TextSelim{The present work has received funding from the European Union’s Horizon 2020 Marie Skłodowska Curie Innovative Training Network Greenedge (GA. No. 953775).} \TextBurak{This work was partially funded by the European Research Council (ERC) through Starting Grant BEACON (no. 677854) and by the UK EPSRC (grant no. EP/T023600/1) under the CHIST-ERA program.}}
}

\author{\IEEEauthorblockN{Selim F. Yilmaz, Burak Hasırcıoğlu, Deniz Gündüz}
}

\maketitle

\vspace{-20pt}
\begin{abstract}
We consider distributed inference at the wireless edge, where multiple clients with an ensemble of models, each trained independently on a local dataset, are queried in parallel to make an accurate decision on a new sample. In addition to maximizing inference accuracy, we also want to maximize the privacy of local models. We exploit the superposition property of the air to implement bandwidth-efficient ensemble inference methods. We introduce different over-the-air ensemble methods and show that these schemes perform significantly better than their orthogonal counterparts, while using less resources and providing privacy guarantees. We also provide experimental results verifying the benefits of the proposed over-the-air inference approach, whose source code is shared publicly on Github.
\end{abstract}


\section{Introduction}

The increasing adoption of Internet-of-Things (IoT) devices results in the collection and processing of massive amounts of mobile data at the wireless edge. Conventional centralized machine learning (ML) methods are impractical for edge applications due to privacy concerns and limited communication resources. Implementing decentralized ML models at the edge solves this issue, and thus, {\it edge learning} and {\it edge inference} have attracted significant attention over the recent years~\cite{chen2021distributed,wu2019machine,lan2021progressive,gunduz2020communicate}. Edge learning aims to train large ML models in a distributed setting, whereas edge inference aims to make inferences in a distributed manner at the edge. 

Although collaborative training at the edge can bring significant advantages, it requires significant coordination and communication across nodes. Moreover, limited wireless resources are a major bottleneck, and noise, interference, and lack of accurate channel state information can prevent or slow down convergence of learning algorithms or results in a reduced accuracy~\cite{chen2021guest}. Therefore, in this paper, we consider collaborative inference using independently trained model at the edge nodes. While a growing body of work studies distributed learning over wireless networks, the literature on distributed wireless inference, particularly using deep learning techniques, is relatively limited~\cite{jankowski2020wireless,jankowski2020joint,shao2021learning}.


We treat the resultant problem as an ensemble inference problem, where the individual hypotheses of the nodes need to be conveyed to the querying server, and combined for the most accurate decision. Ensemble learning methods combine multiple hypotheses instead of constructing a single best hypothesis to model the data~\cite{dietterich2002ensemble}. In ensemble learning, each hypothesis vote for the final decision, where votes can have weights depending on their confidence. It is generally intractable to find the optimal hypothesis, and choosing a model among a set of equally-good models has the risk of choosing the model that has worse generalization performance; however, averaging these models would reduce this risk~\cite{dietterich2002ensemble,bishop1995neural}. Furthermore, weighted or voting based ensemble methods have theoretical guarantees, e.g., expected error of an averaging ensemble of models is not greater than the average of expected errors of the individual models with a mean square objective~\cite{bishop1995neural}.

Privacy is an important concern in all ML applications since the data about individuals can reveal sensitive information about them. In the case of ensemble inference, when the models are queried, their outputs may reveal sensitive information about their training sets. For instance, even when an adversary has black-box access to the models, whether or not a data point is used during training can be inferred via membership inference attacks~\cite{shokri2017membership}, or even the whole model can be reconstructed via model inversion attacks~\cite{tramer2016stealing}. Hence, even if adversaries can only observe the inference results, we need to introduce some additional mechanisms to protect the sensitive information. 

Differential privacy (DP) guarantees can be obtained via introducing additional randomness to the output, such as adding noise at the expense of some accuracy loss. Since DP bounds the amount of information leaked about the individuals, DP mechanisms make black-box attacks less effective. One approach to provide DP guarantees to ML is differentially private training~\cite{abadi2016deep}. Typically, Gaussian noise is added to the gradients during training, where the noise variance is determined according to the desired privacy level. This approach is extended to a federated setting in~\cite{geyer2017differentially}.

In this work, we are interested in enabling distributed inference at the edge while limiting the privacy leakage. One straightforward approach is to train the models in a DP manner. However, in this case, a fixed DP guarantee is achieved, and we cannot operate at different privacy-utility trade-offs during inference, which may be beneficial when serving users with different levels of trustworthiness. Moreover, DP training does not prevent the model stealing attacks since the model can be still reconstructed via black-box access to it. Hence, in this paper, we focus on embedding privacy-preserving mechanisms into the inference phase. We simply lift DP training assumption on the models and assume non-private training.

\TextBurak{In a recent line of work~\cite{amiri2020machine,amiri2020federated,zhu2019broadband}, it has been shown that, in distributed training tasks, over-the-air computation (OAC) can be exploited to use communication resources much more efficiently, and to significantly improve the learning performance. Instead of conventional digital communication, in OAC, clients transmit their updates simultaneously in an uncoded manner such that the receiver automatically gets the aggregated signal.}
\TextSelim{Hence, besides communication efficiency, OAC also helps preserving privacy of the clients.}
\TextBurak{Any noise received simultaneously with the aggregated signal at the receiver is effective at preserving the privacy of all the signals transmitted by the clients \cite{seif2020wireless, liu2020privacy, sonee2020efficient, seif2021privacy, hasirciouglu2021private}.}
In this work, we extend the use of OAC beyond distributed training and exploit it for efficient and private distributed edge inference. 

In particular, we introduce two different ensemble methods along with our private edge inference exploiting OAC. Our main contributions are as follows:
\begin{enumerate}
    \item To the best of our knowledge, this is the first work to employ OAC for distributed inference through an ensemble of models. We show that OAC improves both the privacy and the bandwidth efficiency.
    \item We provide flexible privacy guarantees depending on the scenario without imposing any restrictions on the training phase.
    \item We systematically compare and discuss privacy of the introduced ensemble methods, and show that the proposed framework with OAC performs significantly better than orthogonal counterparts while using less resources.
    \item To facilitate further research and reproducibility, we publicly share the source code of our framework on \href{https://github.com/selimfirat/oac-based-private-ensembles}{github.com/selimfirat/oac-based-private-ensembles}.
\end{enumerate}
\section{System Model and Problem Definition}

\textbf{Notation:} Boldface lowercase letters denote vectors (e.g., $\vec{p}$), boldface uppercase letters denote matrices (e.g., $\vec{P}$), non-boldface letters denote scalars (e.g., $p$ or $P$), and uppercase calligraphic letters denote sets (e.g., $\mathcal{P}$). Blackboard bold letters denote function domains (e.g., $\mathbb{P}$). $\real$, $\mathbb{N}$, $\mathbb{C}$ denote the set of real, natural and complex numbers, respectively. We define $[n]\triangleq\{1,2,\dots,n\}$, where $n\in\mathbb{N}$.

\textbf{System Model:} We consider privacy-preserving ensemble classification at the wireless edge. In this setting, there are $n$ clients each with a separate trained model $f_i:\mathbb{R}^d\to \mathbb{R}^k,i\in[n]$, for a classification task. We assume that local models are trained by using non-intersecting datasets.

We assume that the clients are connected to a central inference server (CIS) via a wireless medium, and, at time $t$, we assume each client $i$ knows its channel gain $h_{i,t}\in\mathbb{C}$. In our setting, the channel gains change across users and time steps, but they stay the same per inference round.
To reduce the total power consumption and to amplify the privacy guarantees, we consider random participation of the clients in each inference round \TextBurak{such that each client $i$ independently participates with probability $p$}. To limit the power consumption, only the clients whose channel gains are larger than a certain threshold participate the inference. This is one of the sources of randomness determining $p$. Hence, $p$ is a tunable parameter via such a transmission threshold. If necessary, via additional randomness, $p$ can be made even smaller. Each participating client makes a prediction denoted by $f_i(\vec{x}_t)$. The clients have a bandwidth of $k$ channel uses to convey their predictions to the CIS.

Let $\vec{y}_{i,t}\in\mathbb{R}^k$ denote the signal transmitted by client $i$. The received signal at CIS is 
\begin{equation}
\vec{z}_t = \sum_{i \in \mathcal{P}_t}  h_{i,t} \vec{y}_{i,t} +  \vec{n}_t,
\end{equation}
where $\vec{n}_t \in \real^k$ is the independently and identically distributed (i.i.d.) additive white Gaussian noise (AWGN) with variance $\sigma^{2}_\mathrm{channel}$, i.e., $\vec{n}_t \sim \mathcal{N}(\vec{0}, \sigma^2_\mathrm{channel} \vec{I}_k)$.

After receiving $\vec{z}_t$, CIS processes it via a function $s:\real^k\rightarrow[k]$ and outputs the most probable class. 

\textbf{Threat Model:} In our problem, the purpose is to limit the privacy leakage of clients' local models. This is equivalent to limiting the leakage about the individual datasets $\mathcal{D}_i,i\in[n]$. In our threat model, we assume all the clients are trusted, i.e., they are not interested in the sensitive features of the training datasets. On the other hand, CIS is honest but curios, i.e., it does not deviate from the protocol, but by using the signals it receives from the clients, it may try to infer sensitive information about the datasets. Hence, our goal is to limit the leakage to CIS about the datasets via $\vec{z}_t$ while trying to maximize the inference accuracy.

\section{Methodology}\label{sec:methodology}
Here, we introduce the modules of our framework gradually, which is summarized in \figref{main_figure}.
\begin{figure*}[t!]
    \centering
    \includegraphics[width=\textwidth]{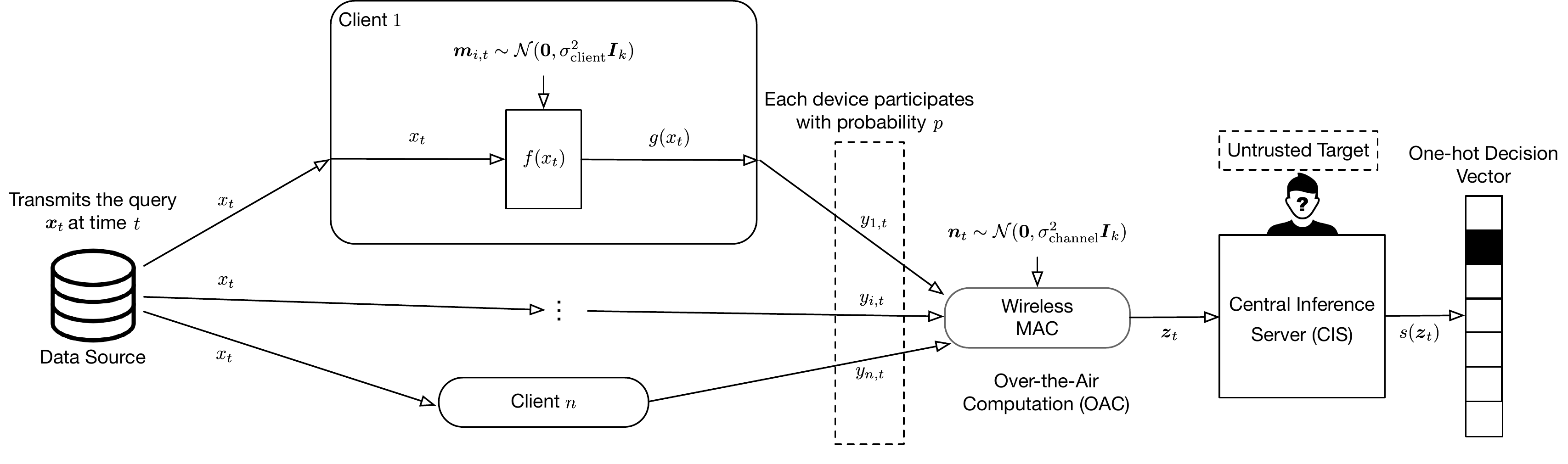}
    \caption{Overview of our ensemble framework for private inference. 
    }
    \label{fig:main_figure}
\end{figure*}

\subsection{Ensemble Methods}
Having received the query $\vec{x}_t$, each participating client makes a local prediction. 
We present alternative ways of doing this by introducing different classes of models, $f_i$'s. Common to all of them, let $\vec{r}_{i,t} \in \real^k$ be a vector containing classifier scores (beliefs) for each class, where $k$ is the number of classes and $j^\mathrm{th}$ element of $\vec{r}_{i,t}$, denoted by $(\vec{r}_{i,t})_j$, contains the score of client $i$ for class $j$. We normalize the sum of the scores in $\vec{r}_{i,t}$ to 1, i.e., $\Vert \vec{r}_{i,t}\Vert_1 = 1$, and hence, the maximum possible score of a class is 1. 

\begin{defn}
$\mathrm{ToOneHot}(j, l)$ function outputs an $l$ dimensional one-hot vector for $j \leq l$, where only the $j^\mathrm{th}$ dimension is $1$ and the rest are $0$.
\end{defn}

{\it Belief summation} method sums beliefs of the participating clients for all the classes and the CIS later selects the class with the highest total score. Thus, it uses the following model for client $i$:
\begin{equation}
    f_i (\vec{x}_t) = \vec{r}_{i,t}.
\end{equation}

{\it Majority voting with OAC} method allows participating clients to vote for a class and the CIS later selects the class with highest number of votes. Hence, it uses the following model for client $i$:
\begin{equation}
    f_i (\vec{x}_t) = \mathrm{ToOneHot} \left( \argmax_{j\in[k]} (\vec{r}_{i,t})_j, k \right).
\end{equation}

Hence, while belief summation with OAC combines local discriminative scores, majority voting with OAC combines predicted labels.







\subsection{Ensuring Privacy}\label{subsec:ensuring-privacy}
Next, we explain how we make our inference procedure privacy-preserving by introducing some randomness. First, we formally define DP for our ensemble inference task as follows.

Let $\mathcal{L}$ and $\mathcal{L}'$ be the sets of local models of the clients, which differ at most in one of the clients, i.e., $\mathcal{L}=\{f_j\}\cup \{f_i : i\in[n]\setminus j\}$ and $\mathcal{L}'=\{f_j'\}\cup \{f_i : i\in[n]\setminus j\}$ such that $f_j\neq f_j'$. Such $\mathcal{L}$ and $\mathcal{L}'$ are called \emph{neighboring} sets. In our case, since we aim to protect the local models from CIS, $\Vec{z}_t$ can be considered as a randomized function, and the set of local models $\mathcal{L}$ or $\mathcal{L}'$ can be considered as its inputs. Hence, all the DP guarantees given in the paper will consider local-model-level privacy guarantees.

\begin{defn}
Let $M:\mathbb{L}\to \real^k$ be a randomized algorithm and $\mathcal{L}$ and $\mathcal{L}'$ are two possible neighboring model sets. For $\varepsilon>0$ and $\delta \in [0,1)$, $M$ is called $(\varepsilon,\delta)$-DP if
\begin{equation}
    \Pr(M(\mathcal{L})\in \mathcal{R}) \leq e^\varepsilon \Pr(M(\mathcal{L}')\in \mathcal{R}) + \delta,
\end{equation}
for all neighboring pairs $(\mathcal{L},\mathcal{L}')$  and $\forall$ $\mathcal{R}\subset \mathbb{R}^k$.
\end{defn}

To achieve DP guarantees, the output released to an adversary should be randomized. In our paper, we consider releasing a noisy version of model outputs, $f_i(\vec{x}_t)$, for each client with a Gaussian noise \cite{dwork2006calibrating}. Note that in OAC, $\vec{z}_t$ already has channel noise\TextBurak{, which provides some degree of privacy guarantees}. However, to achieve the desired level of DP, channel noise may not be large enough and we cannot control or reliably know its variance. Thus, it is not a reliable source of randomness \cite{hasirciouglu2021private}, and we ignore the channel noise while \TextBurak{analysing} privacy guarantees. Instead, we have each client add some additional Gaussian noise before releasing their contributions. \TextBurak{Note that ignoring the channel noise in the privacy analysis results in weaker privacy guarantees. In reality, the privacy guarantees are slightly better than the ones we obtain in this work due to channel noise.} We generate a noisy version of our model prediction as follows.
\begin{equation}
    g_i (\vec{x}_t) = f_i(\vec{x}_t) + \vec{m}_{i,t},
\end{equation}
where $\vec{m}_{i,t} \sim \mathcal{N}(\vec{0}, \sigma^{2}_\mathrm{client} \vec{I}_k)$. One of the main advantages of OAC is that the noise added by different clients are also aggragated at CIS. Thus, it has a further privacy amplification effect. We provide the analysis of the privacy guarantees achieved by our framework in \secref{privacy-analysis}. This analysis reveals that DP guarantees are directly dependent on the variance of the aggregated noise at the CIS. Hence, to obtain DP guarantees, independent of the number of participating clients, each client should add a Gaussian noise with $\sigma_{\mathrm{client}}^2=\sigma^2/|\mathcal{P}_t|$, where $\sigma^2$ is a constant depending on the desired DP guarantees \TextBurak{and $\mathcal{P}_t$ is the set of participating clients}. Hence, we assume that the number of participating clients is known to the other participating clients, but secret from the CIS. 

\subsection{Transmission}

We need to make sure that each client's noisy score $g_i(\vec{x}_t)$ is received at CIS at the same power level. Recall that the channel gain for each client is perfectly known by that client, which then employs channel inversion to cancel its effect. Thus, each client scales the signal by $1/h_{i,t}$. Note that since a client does not participate the inference if its channel has a low gain, this scaling does not result in an excessive power usage. The CIS may require a specific power level for the reception of the signals depending on the available power of the clients. Hence, the clients further scale their signals with a constant denoted by $A_t$, and the transmitted signal is 
\begin{equation}
    \vec{y}_{i,t} = \begin{cases} 
      A_t g_i (\vec{x}_t)/h_{i,t}, & \mathrm{if \,\,}  i\in \mathcal{P}_t   \\ 
      \vec{0}, & \mathrm{otherwise}
      \end{cases}
\end{equation}

\subsection{Final Decision by CIS}

The signal received by the CIS at time $t$ becomes
\begin{equation}
\vec{z}_t = A_t \left( \sum_{i\in\mathcal{P}_t}\vec{m}_{i,t} + \sum_{i \in \mathcal{P}_t} f_i(\vec{x}_t)\right) + \vec{n}_t.
\end{equation}
Thus, the variance of total noise received by the CIS at time $t$ is $\sigma^2_\mathrm{CIS} = \sigma^{2}_\mathrm{channel} + \TextSelim{|\mathcal{P}_t|}  A_t^2  \sigma^{2}_\mathrm{client}$\TextSelim{, i.e., $\vec{z}_t \sim \mathcal{N}(\vec{0}, \sigma^{2}_\mathrm{CIS} \vec{I}_k)$}. After receiving $\vec{z}_t$, CIS multiplies the received signal by $\frac{1}{A_t}$ to recover the desired signal and applies the $\argmax$ function to decide the most probable class. That is, it applies $ s (\vec{z}_t ) = \argmax_{j\in[k]} \frac{1}{A_t} \vec{z}_{t,j}.$

\begin{algorithm}[tb!]
\caption{Over-the-Air Private Ensemble Inference}
\label{algor:model}
\begin{algorithmic}
\Require Trained client model $f_i(\cdot)$ for every client $i$, CIS model $s (\cdot)$, new sample $\vec{x}_t$ at timestep $t$
\Ensure Index of the decided class
\Function{OTA\_Private\_Ensemble}{}
    \State Let $\mathcal{P}_t$ contain each client $i$ with probability $p$, independently
    \ForEach{client $i \in \mathcal{P}_t$ in {\it parallel}}
        \State Client $i$ receives $\vec{x}_t$
        \State Calculate $f_i(\vec{x}_t)$ \Comment{Client Model}
        \State $g_i (\vec{x}_t) = f_i(\vec{x}_t) + \mathcal{N}(\vec{0}, \sigma^{2}/|\mathcal{P}_t| \vec{I}_k)$ \Comment{Add noise}
        \State Transmit $A_t g_i (\vec{x}_t)/h_{i,t}$
    \EndFor
    \State $\vec{z}_t = \vec{n}_t +  A_t \sum_{i\in\mathcal{P}_t}g_i(\vec{x}_t)$ \Comment{Air Sum}
    \State CIS receives $\vec{z}_t$
    \State \Return $s(\vec{z}_t)$ \Comment{CIS Model}
\EndFunction
\end{algorithmic}
\end{algorithm}

\algorref{model} summarizes all the steps introduced in this section.
\section{Privacy Analysis}
\label{sec:privacy-analysis}

In this section, we provide the privacy analysis of the proposed over-the-air ensembling scheme. 
We first analyze the case in which all the clients participate. 

\begin{thm}
\label{thm:eps-del-thm}
If all the clients participate in the inference, i.e., $p=1$, then, \algorref{model} is $(\varepsilon,\delta)$-DP such that for any $\varepsilon>0$,
\begin{equation}
\label{eq:eps-del-thm}
\delta=\Phi(1/(\sqrt{2}\sigma)-\varepsilon\sigma/\sqrt{2})-e^\varepsilon\Phi(-1/(\sqrt{2}\sigma)-\varepsilon\sigma/\sqrt{2}),    
\end{equation}
where $\Phi$ is the CDF of standard normal distribution.
\end{thm}

\begin{proof}
Our theorem is a special case of the following lemma.

\begin{lem}[Theorem 8 in \cite{balle2018improving}]\label{lem:analytical_GM}
Let $f:\mathbb{L}\to \mathbb{R}^k$ be a function with $||f(\mathcal{L})-f(\mathcal{L}')||_2\leq C$, where $\mathcal{L}$ and $\mathcal{L}'$ are neighboring inputs and $||\cdot||_2$ is $L_2$ norm. A mechanism $M(\mathcal{L})=f(\mathcal{L})+\mathcal{N}(0,\tilde{\sigma}^2\Vec{I}_k)$ is $(\varepsilon,\delta)$-DP if and only if 
\begin{equation}
\label{eq:analytical_GM}
    \Phi\left(C/(2\tilde{\sigma})-\varepsilon\tilde{\sigma}/C\right)-e^{\varepsilon}\Phi\left(-C/(2\tilde{\sigma})-\varepsilon\tilde{\sigma}/C\right)\leq \delta.
\end{equation}
\end{lem}

To apply \lemref{analytical_GM} directly in our case, we need to calculate the $L_2$ sensitivity, $C$,  of $\Vec{z}_t$ without any noise, i.e., $\Vec{m}_{i,t}=\Vec{0}, \forall i \in \mathcal{P}_t$. We denote this quantity by $\Vec{\tilde{z}}_t$.
Consider neighboring sets $\mathcal{L}$ and $\mathcal{L}'$. We denote the noiseless vector received by the CIS by $\Vec{\tilde{z}}_t$ when the set of local models is $\mathcal{L}$, and by $\Vec{\tilde{z}}_t'$ when it is $\mathcal{L}'$. Then,
\begin{equation}
    C = \max_{\Vec{\tilde{z}}_t, \Vec{\tilde{z}}_t' } \left\Vert \Vec{\tilde{z}}_t- \Vec{\tilde{z}}_t' \right\Vert_2 = \max_{\Vec{\tilde{z}}_t \Vec{\tilde{z}}_t' } \left ( \sum_{j=1}^{k} (\Vec{\tilde{z}}_{t,j} - \Vec{\tilde{z}}'_{t,j})^2  \right)^{1/2}.
\end{equation}

We know that $\Vec{\tilde{z}}_{t,j} \in [0,A_t], \forall j \in [k]$ and $\Vert \Vec{\tilde{z}}_{t,j} \Vert_1 = A_t$. The same also applies to $\Vec{\tilde{z}}_t'$. Hence, $\Vert \Vec{\tilde{z}}_t- \Vec{\tilde{z}}_t' \Vert_2$ is maximized 
when $\Vec{\tilde{z}}_t$ and $\Vec{\tilde{z}}_t'$ have only one non-zero element, and the indices of these non-zero elements are different in both vectors.
Then, $C=\max_{\Vec{\tilde{z}}_t, \Vec{\tilde{z}}_t' } \Vert \Vec{\tilde{z}}_t- \Vec{\tilde{z}}_t' \Vert_2=\sqrt{2}A_t$.

Finally, by substituting $C=\sqrt{2}A_t$ and $\tilde{\sigma} = \sigma A_t$ into \eqref{analytical_GM}, we obtain \eqref{eps-del-thm}.
\end{proof}

Next, we present the amplification effect of client sampling on the privacy guarantees.

\begin{thm}
\label{thm:eps-del-samp-thm}
If each client independently participate in inference with probability $p<1$, then \algorref{model} is $(\varepsilon',\delta')$-DP, where, for any $\varepsilon'>0$, 
\begin{multline}
\label{eq:eps-del-samp-thm}
\delta'=\frac{p}{1-(1-p)^n}\Big(\Phi(1/(\sqrt{2}\sigma)-\varepsilon\sigma/\sqrt{2})\\
-e^{\varepsilon}\Phi(-1/(\sqrt{2}\sigma)-\varepsilon\sigma/\sqrt{2})\Big),
\end{multline}
where $\varepsilon=\log(1+((1-(1-p)^n)/p)(e^{\varepsilon'}-1))$.
\end{thm}
\begin{proof}
Without loss of generality, let $\mathcal{L}$ and $\mathcal{L'}$ are two neighboring sets of models differing only in the first client's model, i.e. it is either $f_1$ or $f_1'$. Let us write the output distribution of \algorref{model} as mixture distributions. When the model set is $\mathcal{L}$, we have 
$\mu = (1-\eta)\mu_0 + \eta\mu_1$ and when the model set is $\mathcal{L'}$, we have $\mu' = (1-\eta)\mu_0 + \eta\mu_1'$. In these expressions, $\eta$ is the probability that client 1 is sampled, $\mu_0$ is the probability distribution when client 1 is not sampled, $\mu_1$ is the probability distribution when client 1 is sampled and the model set is $\mathcal{L}$ and $\mu_1'$ is the probability distribution when client 1 is sampled and the model set is $\mathcal{L'}$. Recall that we sample client models each with probability $p$ from $\mathcal{L}$ or $\mathcal{L}'$, and the CIS receives non-zero vectors only when $|\mathcal{P}_t|>0$. Hence, $\eta=\Pr\{\text{Client 1 is sampled}\mid |\mathcal{P}_t|>0\}$, resulting in $\eta=p/(1-(1-p)^n)$ via Bayes' rule.

\begin{lem}[Theorem 1 in \cite{balle2018privacy}]
\label{lem:privacy_profile} A mechanism $\mathcal{M}$ is $(\varepsilon',\delta')$-DP
if and only if 
\begin{equation}
\sup_{\mathcal{L},\mathcal{L'}}D_{\alpha}(\mathcal{M}(\mathcal{L})||\mathcal{M}(\mathcal{L'}))\leq\delta',    
\end{equation}
where $\alpha=e^{\varepsilon'}$ and $D_{\alpha}(\mu||\mu')\triangleq\int_{Z}\max \{0,d\mu(z)-\alpha d\mu'(z)\}d(z).$
\end{lem}
\lemref{privacy_profile} implies that it is enough to bound $D_{\alpha}(\mu||\mu')$ to provide DP guarantees. For this, we use the relation in \lemref{ajc}, which is called \emph{advanced joint convexity} of $D_{\alpha}$.

\begin{lem}[Theorem 2 in \cite{balle2018privacy}]
\label{lem:ajc}For $\alpha\geq1$,
we have 
\begin{equation}
\label{eq:advanced_joint_convexity}
D_{\alpha'}\left(\mu||\mu'\right)=\eta D_{\alpha}\left(\mu_{1}||(1-\beta)\mu_{0}+\beta\mu_{1}'\right)
\end{equation}
where $\alpha'=1+\eta(\alpha-1)$ and $\beta=\alpha'/\alpha$. 
\end{lem}

We further upper bound \eqref{advanced_joint_convexity} via convexity:
\begin{equation}\label{eq:convexity}
    D_{\alpha'}\left(\mu||\mu'\right) \leq \eta (1-\beta)D_{\alpha}(\mu_1||\mu_0) + \eta \beta D_{\alpha}(\mu_1||\mu_1').
\end{equation}

To bound $D_{\alpha}(\mu_1||\mu_0)$, observe that there exist a coupling between $\mu_1$ and $\mu_0$ as follows. For $\mu_0$, to guarantee $|\mathcal{P}_t|>0$, let us first sample exactly one client $c$ other than client 1 since we know that client 1 is not sampled. Then apply Poisson sampling on the remaining set, i.e., $[n]\setminus\{c,1\}$, to determine the other participating clients. For $\mu_1$, assume we have the same realization of Poisson sampling on $[n]\setminus \{c,1\}$ as in $\mu_0$. Further, by definition of $\mu_1$, client 1 is also sampled. Hence, $\mu_1$ and $\mu_0$ can be seen as output distributions of \algorref{model} such that the input client sets differ in only one element. Hence, $D_{\alpha}(\mu_1||\mu_0)\leq \delta$ due to \thmref{eps-del-thm}. Similarly, to bound $D_{\alpha}(\mu_1,\mu_1')$, a coupling exists between $\mu_1$ and $\mu_1'$ such that user 1 is sampled and $f_1$ and $f_1'$ are the models in user 1, for $\mu_1$ and $\mu_1'$, respectively. To determine the other participating clients, the same realization of Poisson sampling on $[n]\setminus\{1\}$ is applied in both $\mu_1$ and $\mu_1'$. Since the input client sets also differ in one element, in this case, due to \thmref{eps-del-thm}, we have $D_{\alpha}(\mu_1,\mu_1')\leq \delta$. If we put the bounds for $D_{\alpha}(\mu_1,\mu_0)$ and $D_{\alpha}(\mu_1,\mu_1')$ into \eqref{convexity}, we obtain $D_{\alpha'}(\mu||\mu')\leq \eta \delta$, from which \eqref{eps-del-samp-thm} follows. The expression for $\varepsilon$ can be directly derived from the expression $\alpha'=1+\eta(\alpha-1)$.
\end{proof}


\section{Simulations}


\subsection{The Datasets and Experimental Setup}
We employ four different datasets to demonstrate the effectiveness of our framework: CIFAR-10, CIFAR-100, FashionMNIST and IMDB. CIFAR-10 contains $50.000$ training images, $10.000$ test images, and 10 target classes~\cite{krizhevsky2009learning}.
CIFAR-100 contains the same splits except that target classes are partitioned into 100 subclasses~\cite{krizhevsky2009learning}. FashionMNIST has $50.000$ training images, $10.000$ test images, and 10 target classes~\cite{xiao2017/online}. IMDB dataset has $25.000$ training texts, $25.000$ test texts, and $2$ target classes~\cite{maas-EtAl:2011:ACL-HLT2011}. For all datasets, we use predefined training and test sets, except that we split 10\% of the training set as the validation set and only use the remaining 90\% for training.

For image datasets, we use MobileNetV3-Large~\cite{howard2019searching} except we change its final layer to make it compatible with the target number of classes. Instead of training from scratch, we fine-tune a pre-trained version \cite{marcel2010torchvision} of it for 50 epochs. To make sizes of the images compatible to our network, we interpolate them to $224 \times 224$ images. Since the network receives three channel inputs, for each FashionMNIST sample, we feed the same single channel grayscale image to all input channels. For text datasets, we use DistilBERT-base-uncased~\cite{sanh2019distilbert} model, and again, we fine-tune a pre-trained model \cite{wolf-etal-2020-transformers} for 3 epochs. 

We repeat all the experiments with 5 different random seeds, and report the average results. We compute and report Macro-F1 scores by averaging per-class F1 scores on the test set. We randomly split the training data among the clients equally. We consider $n=20$ clients with a participation probability of $p=1.0$ and a channel signal-to-noise ratio (SNR) of $10$ dB, except when they are changed gradually in \secref{conditions}.

\subsection{Comparison with the Baselines}
In Table~\ref{tab:comparison}, in terms of their Macro-F1 scores, we compare the proposed OAC-based methods with the best client model and the ensemble methods with orthogonal transmission. We choose the model with the highest Macro-F1 score on the same validation set as the best client model. For fairness, the client having the best model transmits its inference over the $k$ channels. In orthogonal methods, all the devices transmit their inferences via different channels, i.e., $\card{\mathcal{P}_t} \times k \in O(nk)$ channels in total. 
\begin{table}[t]
    \centering
    \caption{Comparison with the Baselines}
    \resizebox{0.8\textwidth}{!}{\begin{tabular}{clcccc}
    \toprule
    Privacy & Method & CIFAR-10 & CIFAR-100 & FashionMNIST & IMDB\\ \midrule
\multirow{5}{*}{$\epsilon=\infty$} & Best Client Model & $86.37 {\scriptstyle \pm 0.33}$ & $44.73 {\scriptstyle \pm 1.60}$ & $89.55 {\scriptstyle \pm 0.23}$ & $89.31 {\scriptstyle \pm 0.31}$\\
& Orthogonal Majority Voting & $89.97 {\scriptstyle \pm 0.14}$ & $62.51 {\scriptstyle \pm 0.74}$ & $\mathbf{91.92 {\scriptstyle \pm 0.15}}$ & $90.59 {\scriptstyle \pm 0.06}$\\
& Orthogonal Belief Summation & $90.09 {\scriptstyle \pm 0.12}$ & $\mathbf{63.85 {\scriptstyle \pm 0.60}}$ & $91.91 {\scriptstyle \pm 0.11}$ & $\mathbf{90.64 {\scriptstyle \pm 0.05}}$\\
& Majority Voting with OAC & $89.96 {\scriptstyle \pm 0.14}$ & $62.55 {\scriptstyle \pm 0.67}$ & $\mathbf{91.92 {\scriptstyle \pm 0.13}}$ & $90.62 {\scriptstyle \pm 0.10}$\\
& Belief Summation with OAC & $\mathbf{90.14 {\scriptstyle \pm 0.16}}$ & $63.83 {\scriptstyle \pm 0.59}$ & $91.91 {\scriptstyle \pm 0.13}$ & $\mathbf{90.64 {\scriptstyle \pm 0.07}}$\\
\midrule
\multirow{5}{*}{$\epsilon=1$} & Best Client Model & $12.19 {\scriptstyle \pm 0.22}$ & $1.20 {\scriptstyle \pm 0.05}$ & $12.29 {\scriptstyle \pm 0.30}$ & $53.58 {\scriptstyle \pm 0.32}$\\
& Orthogonal Majority Voting & $22.59 {\scriptstyle \pm 0.10}$ & $2.41 {\scriptstyle \pm 0.14}$ & $23.43 {\scriptstyle \pm 0.55}$ & $65.29 {\scriptstyle \pm 0.23}$\\
& Orthogonal Belief Summation & $22.22 {\scriptstyle \pm 0.13}$ & $2.22 {\scriptstyle \pm 0.12}$ & $23.30 {\scriptstyle \pm 0.55}$ & $64.94 {\scriptstyle \pm 0.24}$\\
& Majority Voting with OAC & $\mathbf{81.27 {\scriptstyle \pm 0.10}}$ & $\mathbf{24.24 {\scriptstyle \pm 0.21}}$ & $\mathbf{84.18 {\scriptstyle \pm 0.21}}$ & $\mathbf{89.32 {\scriptstyle \pm 0.07}}$\\
& Belief Summation with OAC & $80.13 {\scriptstyle \pm 0.24}$ & $20.04 {\scriptstyle \pm 0.24}$ & $83.81 {\scriptstyle \pm 0.22}$ & $89.15 {\scriptstyle \pm 0.08}$\\
\bottomrule
    \end{tabular}}
    \label{tab:comparison}
    \vspace{-10pt}
    \end{table}
We observe that, compared to the best client model, ensemble methods significantly improve the test scores, especially in the private setting. Moreover, while orthogonal and OAC-based methods perform competitively in the non-private setting, when privacy is involved, best client model and orthogonal methods perform near-random, and significantly worse than the OAC-based methods. \TextSelim{Note that orthogonal methods use $\card{\mathcal{P}_t} \times k$ channels, whereas OAC-based methods only use $k$ channels; yet, OAC-based methods outperform orthogonal ones in the private setting.}

Previous studies suggest that ensembling via belief averaging generally performs better than majority voting~\cite{kuncheva2002theoretical,wang2020averaging}. Our non-private results also support this argument as beliefs contain more information compared to conveying local decisions. However, when $\varepsilon=1$, majority voting outperforms belief summation for both orthogonal and OAC-based settings. This can be explained by the fact that the increasing noise levels result in relatively unreliable beliefs, since the individual values of beliefs are smaller, and thus more sensitive to the noise added for privacy.

\subsection{Analysis of Ensembles with OAC for Varying Conditions}
\label{sec:conditions}
\figref{conditions} shows the performance of our OAC-based methods on CIFAR-10 dataset for varying channel SNR, $p$, and $\varepsilon$ values. The left figure shows that the performance of the methods slightly increases as the channel SNR increases, especially for SNR values below 2 dB. In the right figure, we observe that higher $p$ improves the performance significantly in the private setting $(\varepsilon=1)$. Although lower $p$ has a privacy amplification effect which decreases the noise variance required to attain $\varepsilon=1$, we observe that its privacy amplification effect is not as significant as the impact of a fewer client participation on the inference performance. 
In the non-private setting $(\varepsilon=\infty)$, having higher participation also helps to get higher macro-F1 score, but not as much as in the private setting. These plots also show that private setting is more sensitive to these varying conditions for both $p$ and channel SNR.

    \pgfplotsset{footnotesize,samples=10}    \definecolor{color0}{rgb}{0.12156862745098,0.466666666666667,0.705882352941177}
    \definecolor{color1}{rgb}{1,0.498039215686275,0.0549019607843137}
    \definecolor{color2}{rgb}{0.172549019607843,0.627450980392157,0.172549019607843}
    \definecolor{color3}{rgb}{0.83921568627451,0.152941176470588,0.156862745098039}
    \definecolor{color4}{rgb}{0.580392156862745,0.403921568627451,0.741176470588235}
    \definecolor{color5}{rgb}{0.580392156862745,1.0,0.741176470588235}

    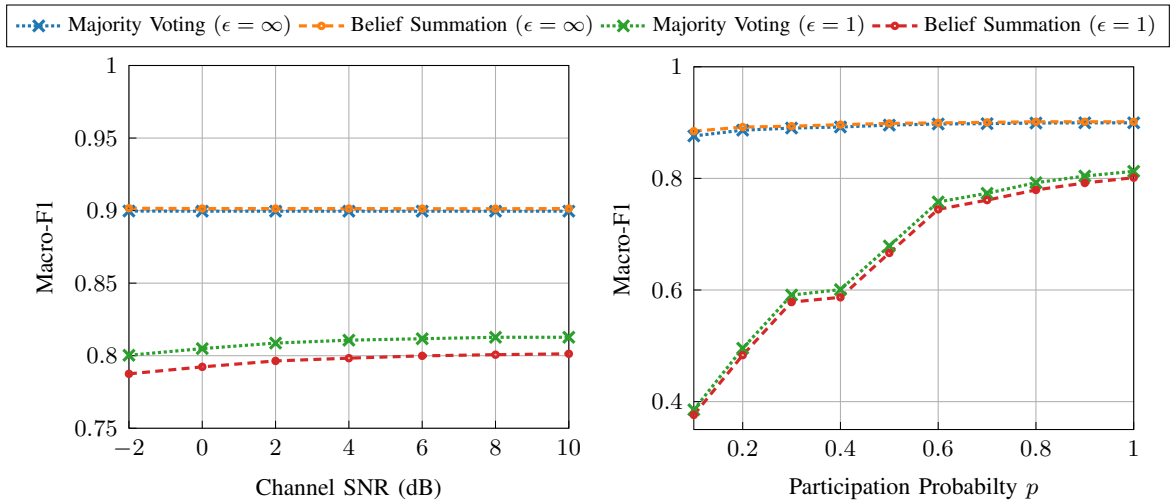
\begin{figure}[t]
    \centering
    \ref*{named}
    \begin{tikzpicture}
    \pgfplotsset{%
    width=.45\textwidth
    }
    \begin{axis}[
    legend columns=4,
    legend entries={\scriptsize Majority Voting $(\epsilon=\infty)$,\scriptsize Belief Summation $(\epsilon=\infty)$,\scriptsize Majority Voting $(\epsilon=1)$,\scriptsize Belief Summation $(\epsilon=1)$},
    legend to name=named,
    log basis x={10},
    log basis y={10},
    tick align=inside,
    tick pos=left,
    x grid style={white!69.0196078431373!black},
    xmin=-2, xmax=10,
    xmode=linear,
    xtick style={color=black},
    y grid style={white!69.0196078431373!black},
    ymin=0.75, ymax=1.0,
    ymode=linear,
    ytick style={color=black},xlabel={\footnotesize Channel SNR (dB)},
    ylabel={\footnotesize Macro-F1},
    max space between ticks=1000pt,
    try min ticks=5,
    grid
    ]
        
        \addplot [densely dotted, very thick, color0, mark=x, mark size=3, mark options={solid}]
        table {
-2	0.8997
0	0.8996
2	0.8996
4	0.8996
6	0.8996
8	0.8996
10	0.8996
12	0.8996
14	0.8996
16	0.8996
18	0.8996
20	0.8996
};
        
        \addplot [densely dashed, very thick, color1, mark=o, mark size=1, mark options={solid}]
        table {
-2	0.9016
0	0.9014
2	0.9014
4	0.9014
6	0.9013
8	0.9013
10	0.9014
12	0.9014
14	0.9013
16	0.9013
18	0.9013
20	0.9013
};
        
        \addplot [densely dotted, very thick, color2, mark=x, mark size=3, mark options={solid}]
        table {
-2	0.8003
0	0.8049
2	0.8087
4	0.8107
6	0.8117
8	0.8127
10	0.8127
12	0.8129
14	0.8129
16	0.8131
18	0.8132
20	0.8134
};
        
        \addplot [densely dashed, very thick, color3, mark=o, mark size=1, mark options={solid}]
        table {
-2	0.7875
0	0.7923
2	0.7964
4	0.7983
6	0.7999
8	0.8007
10	0.8013
12	0.8013
14	0.8016
16	0.8017
18	0.8017
20	0.8019
};
        
        \end{axis}
    \end{tikzpicture}
    \begin{tikzpicture}
    \pgfplotsset{%
    width=.45\textwidth
    }
    \begin{axis}[
    log basis x={10},
    log basis y={10},
    tick align=inside,
    tick pos=left,
    x grid style={white!69.0196078431373!black},
    xmin=0.1, xmax=1.0,
    xmode=linear,
    xtick style={color=black},
    y grid style={white!69.0196078431373!black},
    ymin=0.35, ymax=1.0,
    ymode=linear,
    ytick style={color=black},xlabel={\footnotesize Participation Probabilty $p$},
    ylabel={\footnotesize Macro-F1},
    max space between ticks=1000pt,
    try min ticks=5,
    grid
    ]
        
        \addplot [densely dotted,very thick, color0, mark=x, mark size=3, mark options={solid}]
        table {
0.10	0.8761
0.20	0.8864
0.30	0.8902
0.40	0.8921
0.50	0.8955
0.60	0.8975
0.70	0.8981
0.80	0.8992
0.90	0.8996
1.00	0.8996
};
        
        \addplot [densely dashed,very thick, color1, mark=o, mark size=1, mark options={solid}]
        table {
0.10	0.8846
0.20	0.8920
0.30	0.8935
0.40	0.8963
0.50	0.8985
0.60	0.8995
0.70	0.9004
0.80	0.9015
0.90	0.9015
1.00	0.9014
};
        
        \addplot [densely dotted,very thick, color2, mark=x, mark size=3, mark options={solid}]
        table {
0.10	0.3854
0.20	0.4956
0.30	0.5910
0.40	0.6007
0.50	0.6789
0.60	0.7578
0.70	0.7733
0.80	0.7923
0.90	0.8042
1.00	0.8127
};
        
        \addplot [densely dashed,very thick, color3, mark=o, mark size=1, mark options={solid}]
        table {
0.10	0.3762
0.20	0.4833
0.30	0.5784
0.40	0.5869
0.50	0.6664
0.60	0.7445
0.70	0.7611
0.80	0.7793
0.90	0.7918
1.00	0.8013
};
        
        \end{axis}
    \end{tikzpicture}
    \caption{Comparison of ensemble methods with OAC for varying channel SNR (left) and participation probability $p$ (right) on CIFAR-10 dataset.}
    \label{fig:conditions}
    \vspace{-8pt}
    \end{figure}

\section{Conclusion}
We have introduced a private edge inference framework with ensembling. We have exploited OAC for bandwidth-efficient and private wireless edge inference for the first time in the literature. We have provided DP guarantees exploiting both distributed noise addition and random participation. We have systematically evaluated the introduced ensemble methods with OAC and shown that distributed edge inference with OAC performs significantly better than its orthogonal counterpart while using less resources. We have observed that while transmitting class scores from each client is more informative as an ensembling method, making and transmitting local decisions can be more reliable when noise is introduced to guarantee privacy.

\bibliography{ref.bib}
\bibliographystyle{IEEEtran}

\end{document}